%% file: IJCAI.tex
\documentclass{article}
\usepackage{ijcai26}

\usepackage{times}
\usepackage{soul}
\usepackage{url}
\usepackage[hidelinks]{hyperref}
\usepackage[utf8]{inputenc}
\usepackage[small]{caption}
\usepackage{graphicx}
\usepackage{amsmath}
\usepackage{amsthm}
\usepackage{booktabs}
\usepackage{algorithm}
\usepackage[switch]{lineno}
\usepackage{algpseudocode}
\usepackage{amssymb}
\usepackage{booktabs}
\usepackage{colortbl}
\usepackage{xcolor}

\usepackage{subcaption}

\theoremstyle{definition}
\newtheorem{definition}{Definition}

\theoremstyle{plain}
\newtheorem{theorem}{Theorem}[section]
\newtheorem{lemma}{Lemma}[section]

\definecolor{best}{gray}{0.92}

\title{COREKG: Coreset-Guided Personalized Summarization of Knowledge Graphs}

\author{
Sohel Aman Khan \and
Raghava Mutharaju\thanks{Raghava Mutharaju is currently affiliated with the Mehta Family School of Data Science and AI, IIT Palakkad, India, and can be contacted at raghava@iitpkd.ac.in.} \and
Supratim Shit \\
\affiliations
Department of CSE, IIIT-Delhi, India \\
\emails
\{sohelk, raghava.mutharaju, supratim\}@iiitd.ac.in
}

\begin{document}

\maketitle


\begin{abstract}
Knowledge Graphs (KGs) are extensively used across different domains and in several applications. Often, these KGs are very large in size. Such KGs become unwieldy for tasks such as question answering and visualization. Summarization of KGs offers a viable alternative in such cases. Furthermore, personalized KG summarization is crucial in the current data-driven world as it captures the specific requirements of users based on their query patterns. Since it only maintains relevant information, the personalized summaries of KG are small, resulting in significantly smaller storage requirements and query runtime. In this work, we adapt the coreset theory to create personalized KG summaries. For a given dataset and a user-specific query workload, we present an approach that samples a relevant subset of triples using sensitivity-based importance sampling. We ensure that the subset approximates the characteristics of the full dataset with bounded approximation error. We define sensitivity scores that measure the importance of a triple with respect to a user’s query workload, which are then used by our coreset construction algorithm. We explicitly focus on personalized knowledge graph summarization by constructing summaries independently for each user based on their query behaviour. 
Our evaluation on Freebase, WikiData, and DBpedia shows that COREKG delivers higher query-answering accuracy and structural coverage than the state-of-the-art methods, such as GLIMPSE, PPR, iSummary, PEGASUS and APEX$^2$ while requiring only a tiny fraction of the original graph. 
\end{abstract}
  
\input{sections/introduction}

\input{sections/Related_Works}

\input{sections/Preliminaries}

\input{sections/Methodology}

\input{sections/Experimants}

\newpage
\appendix
\input{sections/Appendix}

\bibliographystyle{named}
\bibliography{ref}

\end{document}

%% file: sections/introduction.tex
\section{Introduction}
In the digital age, numerous sources of information are readily available. Knowledge Graphs (KGs)~\cite{peng2023knowledge} are an effective way to represent and reason over interconnected data. Large-scale KGs such as Wikidata, DBpedia, and Freebase~\cite{bollacker2008freebase,mendes2012dbpedia} have become central to numerous real-world applications, including search engines, recommendation systems, question-answering (QA), and biomedical informatics. However, the vast size and high connectivity of these KGs introduce challenges for efficient querying, visualization, and reasoning. QA systems often struggle with slow query execution, and visualization tools become cluttered and less interpretable at scale~\cite{wang2024survey}. Knowledge graph summarization addresses these issues by helping users understand and utilize the data while preserving key structural and semantic information. However, most summarization methods are focused on capturing the general pattern of the original KG~\cite{jalota2021lauren,wang2024survey}. As a result, the generated summaries tend to be overly generalized and often fail to accommodate the specific information needs of individual users. 

Query-based knowledge graph summarization addresses this issue by constructing compact summaries optimized for a given query workload rather than the global graph structure. It improves query efficiency and answer completeness under reduced graph sizes \cite{song2018mining,niazmand2022efficient}. However, such workloads are typically defined at an aggregate level. These summaries do not capture individual user preferences, motivating the need for personalized knowledge graph summarization \cite{jalota2021lauren}. 

Considering the above limitation, personalized summarization of KGs has massive potential because it focuses on user interest. It is also efficient in terms of storage and time. Early research like GLIMPSE~\cite{safavi2019personalized} uses submodular maximization to retrieve user-specific KG summaries using query logs, which enable device-level KG summaries with theoretical guarantees. However, GLIMPSE assumes full access to each user's query history, which has less explicit control over summary sparsity. ISummary~\cite{vassiliou2023isummary}  creates efficient, personalized Knowledge Graph summaries by analyzing SPARQL query logs to capture collective user interests, reducing computation while improving scalability and relevance. However, the summarization quality in iSummary depends heavily on workload coverage. 
Furthermore, the sample size of COREKG ensures controlled budget–risk tradeoffs.
We present a coreset-based framework named COREKG for personalized Knowledge Graph summarization that addresses the limitations of existing personalized approaches~\cite{feldman2017coresets,safavi2019personalized}. By adapting coreset theory to the KG domain, our method provides theoretical guarantees while maintaining computational efficiency~\cite{braverman2021efficient}. Our approach generates compact, user-centric KG summaries by sampling triples based on sensitivity scores that capture their importance across user queries~\cite{safavi2019personalized}. In constructing personalized KG summaries, we focus on queries that contain seed nodes (entities of direct user interest)~\cite{safavi2019personalized} and we provide theoretical bounds that preserve the original query semantics. A high-level overview of the COREKG pipeline, including preprocessing, sensitivity computation, and sampling, is illustrated in the Figure~\ref{fig:diagram}.

Coresets are a form of weighted summarization of a dataset with provable guarantees ~\cite{feldman2011unified}. In the past couple of decades, coresets have been studied for various problems and applications, including clustering ~\cite{feldman2013turning,bachem2018scalable,chhaya2020online}, regression ~\cite{he2021secure,chhaya2020coresets,boutsidis2013near}, multilinear algebra ~\cite{dasgupta2009sampling,chhaya2020streaming}, non-decomposable functions \cite{malaviya2024simple}, federated learning \cite{huang2022coresets,shit2025improved} etc. The most popular method for constructing a coreset is importance sampling via a sensitivity framework, which assigns an importance score to data points (triples) commensurate with their contribution to a particular query workload, thereby supporting efficient importance-based sampling. Some works~\cite{feldman2017coresets,braverman2021efficient} prove these ideas yield strong approximation bounds. COREKG extends these ideas to graph triples, and its coreset summary approximates the utility of a complete graph for personalized queries.

To fulfill our objective, we construct a coreset—a smaller, weighted subgraph—from the original knowledge graph, tailored specifically to a user's query workload. Instead of evaluating every possible query association across the complete graph, we sample a subset of relevant triples and assign them statistical weights to correct for sampling probability. This ensures our coreset provides a provable approximation guarantee; it preserves the query-answering semantics of the full graph up to a small, bounded relative error while drastically reducing the data footprint. Our empirical evaluation further demonstrates superior performance in query coverage over real-world KG datasets~\cite{vassiliou2023isummary,safavi2019personalized}, and the method is effective even with limited computational resources~\cite{safavi2019personalized}.

To explicitly capture per-user preferences, COREKG constructs user-centric workloads based on seed entities that represent an individual user’s interests. These seed entities are extracted from the query workload, and the summary is generated only from queries that mention at least one of these seeds. By repeating this procedure for different seed sets, COREKG naturally adapts to multiple users, enabling personalized knowledge graph summarization rather than relying on a single global workload. Our main contributions are as follows.
\begin{itemize}
    \item We propose COREKG, a coreset-guided framework for personalized knowledge graph summarization that adapts coreset theory to graph data and provides provable approximation guarantees (Theorem \ref{theorem:main}) with practical scalability. 
    \item We introduce a sensitivity-based importance sampling formulation that quantifies the relevance of triples with respect to user-specific query workloads, ensuring robustness across diverse personalization.
    \item Our coreset construction provides a worst-case approximation guarantee on the user-specific workload cost. It is independent of the size of the full graph. All query evaluation on the summary is performed on the coreset.
    \item We conduct extensive experiments on Freebase, DBpedia, and Wikidata, showing that COREKG consistently achieves higher coverage and F1 Score than baseline methods such as APEX$^2$, GLIMPSE, PEGASUS, iSummary, and PPR.
\end{itemize}

%% file: sections/Related_Works.tex
\section{Related Work}
In this section, we review prior studies and approaches that are closely related to our work. We highlight existing methods for knowledge graph summarization, query-driven graph reduction and discuss how COREKG builds upon and differs from these efforts.



Summarization refers to generalizing KGs by removing unnecessary information by merging similar nodes and edges into supernodes and superspecs to generalize the information in the graph. This process removes unnecessary information while preserving the generalization of the information in the graph ~\cite{jalota2021lauren,kumar2018utility}. Most existing summarization approaches primarily emphasize capturing the overall structure of the original KG~\cite{jalota2021lauren,wang2024survey}.

To fix these limitations, Query-based summarization of knowledge graphs aims to construct compact summaries that maximize response completeness and relevance for a workload of specified target queries. Early research demonstrated that lossy summaries can improve query processing time performance, and enable efficient KG search capacity through summaries, especially relevant to query task needs ~\cite{song2018mining}. Approaches, such as Grouping-Based and Query-Based Summarization, utilize formal methods that enable the significant compression of graph size while retaining complete query performance ~\cite{niazmand2022efficient}. Lastly, a few relatively recent works are looking at attention, or sampling methods to approach scale issues and maintain query workload performance improvements ~\cite{shabani2024graphsum}. Other lines of work explore the structure of the KG, with works on incremental summarization, such as RDFQuotient, enabling the dynamic generation of summaries without the need to resummarize static RDF datasets ~\cite{goasdoue2019incremental}. Summarization oriented towards queries has similarly been explicitly positioned for downstream tasks like question answering (QA). For example, in QA, summarization can be used as either a filtering method for relevance or to rank KG summaries while achieving comparable QA efficacy ~\cite{jalota2021lauren}. More generally, summarization has been positioned as an enabler for efficient graph mining and querying ~\cite{koutra2021power}. Complementary frameworks also speak to how graph-based summarization techniques could be helpful in support of an abstractive summarization or knowledge extraction pipeline beyond query efficiency ~\cite{moro2023graph}. All of these studies contribute evidence that query-aware graph summarization is a unique area in KG summarization in which both compactness and completeness can be evaluated. But, this summarization can't fulfill user-specific needs. To fix this problem, KGs are summarized by considering user-specific needs. This summarization is known as personalized KG summarization.

GLIMPSE ~\cite{safavi2019personalized} develops a submodular optimization theory for the personalized KG summarization for data from query logs. iSummary ~\cite{vassiliou2023isummary} proposed summarization through workloads that generalize across aggregate user behaviours. Yet, limitation exist for both studies. GLIMPSE assumes each individual has a dense set of histories for generating results, and iSummary relies on the conclusive coverage of SPARQL workloads. The PEGASUS ~\cite{kang2022personalized} method generates summary graphs for individuals in linear time and uses a personalized cost function to generate summaries through greedily combining supernodes to act as underlying groups that allows for expected query answering under a high compression, however it is still a heuristic, which lacks slight approximation, and it assumes the target set of nodes are static, reducing its applicability in rapidly changing and dynamic environments, where direction and interests are constantly evolving. The APEX$^2$ ~\cite{li2025apex2} method modulates personalized assistance in Knowledge Graphs incrementally while modeling a more diverse set of evolving query workloads. It generates summaries in real time based on modeled heat diffusion, while adhering to tight constraints that ensure provable guarantees. This is achieved under the circumstance where the adaptive current workload generated by the querying node provides for the activity of interactions. However, APEX$^2$ ~\cite{li2025apex2} includes complete reliance on representative query workloads, and can fail on queries that are sparse or extremely skewed. APEX$^2$ also assumes that modeling heat diffusion behaves as expected between the workload nodes based on topological proximity, and does not expect semantic relationships as captured in modeling KGs. COREKG circumvents both challenges and gaps by utilizing coreset sampling to provide theoretical guarantees while accommodating sparse or skewed query logs.

Recommendation systems usually personalize suggestions based on user preferences that can be inferred from query histories. For example, model behaviour studies, such as ~\cite{vcebiric2015query,cheng2011relin,gunaratna2015faces}, are designed to characterize user interests using SPARQL queries to extract relevant entities and paths to create personalized summaries. Additionally, some diversity-aware methods, including FACES ~\cite{gunaratna2015faces}, are designed to balance informative against representation diversity in entity-centric summaries.

Personalized PageRank (PPR)~\cite{haveliwala2003analytical} presents a method for designing the ranking method for user-centered exploration utilizing the graph. However, these methods often do not provide any explicit controls for summary size or utility. The COREKG includes an advancement over these approaches by introducing a generalized sampling framework that incorporates user-specific queries to lay the groundwork for explicit personalizability capabilities.

%% file: sections/Preliminaries.tex
\section{Preliminaries}
The preliminaries here provide the foundation for understanding our methodology and analysis.

\noindent
\textbf{Knowledge Graph (KG):} A Knowledge Graph (KG)~\cite{peng2023knowledge} is a structured representation of data, where entities are modeled as nodes and their relationships as labeled edges. Formally, a KG is defined as \( G = (V, E, T) \), where \( V \) is the set of nodes (entities), \( E \) is the set of edge labels (relations), and \( T \subseteq V \times E \times V \) is the set of RDF triples representing relational facts. SPARQL is commonly used to access and manipulate these triples, enabling complex pattern matching and information retrieval. The query workload on a KG consists of sets of SPARQL queries that reflect typical information needs, varying in complexity, frequency, and the types of operations they perform.

\noindent
\textbf{Seed node:} A seed node is an entity that represents a user’s interest and is used to guide personalization. In this paper, Seed nodes are used to filter the global query log to construct a user-specific query workload.

\noindent
\textbf{Personalized KG summarization:} A Personalized Knowledge Graph summary~\cite{safavi2019personalized,vassiliou2023isummary,kang2022personalized,li2025apex2} refers to a summary extracted from a knowledge graph that is tailored to a user’s interests and preferences.

In this setting, seed nodes represent a user’s interests by identifying the entities around which personalization is performed. In our work, seed entities define user-specific query workloads that guide the selection of relevant triples. As a result, the generated KG summary focuses on the parts of the graph that are most important to the user’s information needs~\cite{vassiliou2023isummary}. Example: If a user frequently queries about \textit{workplaces} and \textit{industries}, the chosen seeds may include entities such as \textit{Bob} and \textit{XCorp}. A personalized summary will then prioritize triples related to these seeds, such as: (Bob, worksAt, XCorp) and (XCorp, industry, Tech)

\noindent
\textbf{Coreset:} A Coreset~\cite{bachem2017practical} is a small, weighted subset of data that approximates a cost function with provable error bounds. We adapt this concept to Knowledge Graphs. Given a KG $G = (V,E,T)$ with triples $T$, and a query workload $Q$, we define the cost of answering $Q$ on the full graph as $\text{cost}(G, Q) = \sum_{t \in T} \sum_{q \in Q} f_q(t)$, where $f_q(t) \in {0,1}$ indicates whether triple $t$ is relevant to query $q$.

A coreset $C = (V', E', T')$, where $V' \subseteq V$, $E' \subseteq E$, and $T' \subseteq T$, is a weighted subset of triples. The coreset is constructed by sampling triples according to their importance and then assigning them appropriate weights. Each triple $t \in C$ receives a weight $w(t)$ that compensates for its sampling probability, ensuring unbiased estimation of the full cost. The weighted cost over the coreset is defined as $\text{cost}(C, Q) = \sum_{t \in C} \sum_{q \in Q} w(t) \cdot f_q(t).$ Formally, if each triple $t$ is sampled with probability $p(t)$ and the total number of samples is $m$, the weight is $w(t) = \frac{1}{m \cdot p(t)}.$ This weight $w(t)$ ensures that the expected contribution of each triple in $C$ matches its true contribution in $G$, thereby removing sampling bias. Intuitively, triples that are sampled with lower probability receive higher weights, since they represent a larger portion of the original graph, whereas frequently sampled triples receive smaller weights. The coreset satisfies $\big|\text{cost}(G,Q) - \text{cost}(C,Q)\big| \leq \varepsilon \cdot \text{cost}(G,Q)$ with this weighting and high probability. This definition is generic and applies to any query workload $Q$. In the COREKG framework, it is applied to personalized query workloads derived from individual users, resulting in a user-specific coreset.

%% file: sections/Methodology.tex
\section{Methodology}
We present COREKG (Figure~\ref{fig:diagram}), a coreset-based importance sampling approach for personalized knowledge graph summarization. The method produces compact, query-relevant subgraphs that approximate the behaviour of the full graph with respect to a user’s query workload while guaranteeing bounded approximation error. Here, triples are scored based on how frequently or critically they contribute to a user-specific query workload, where more relevant triples are assigned higher sampling probabilities to preserve informative structure.  COREKG employs sensitivity-based sampling to prioritize triples by their explicit contribution to query answering, ensuring preservation of user-relevant structural and semantic information. It operates as a per-user summarization framework, where user modeling, workload construction, sensitivity estimation, and coreset sampling are performed independently for each user and conditioned on their personalized query workload.

\begin{figure}[t]
    \centering
    \includegraphics[width=\columnwidth]{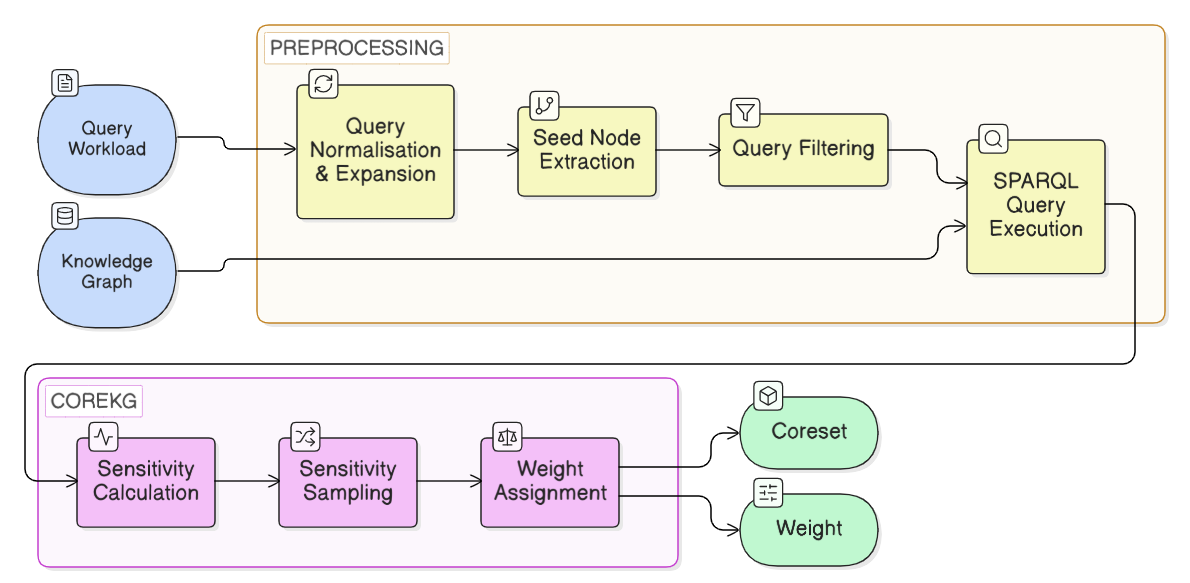}
    \caption{\scriptsize{COREKG Framework}}
    \label{fig:diagram}
\end{figure}

\subsection{Preprocessing Pipeline}


In COREKG, preprocessing provides user modeling by representing each user through a small set of seed entities extracted from query history. These seeds define a personalized query workload that filters relevant queries from the global log and enables user-specific knowledge graph summarization. The preprocessing pipeline comprises four stages: query normalization and expansion, seed extraction, query filtering, and SPARQL execution.

\noindent
\textbf{Query Normalization and Expansion:} In order to maintain semantic consistency within the query workload, we need to first normalize all SPARQL queries. Normalization consists of two stages: First, we take the original SPARQL query and eliminate duplicate PREFIX declarations to remove unnecessary repeated boilerplate in the process. Second, we expand all prefixed terms to complete URIs using a fixed prefix map via regular expression matching and replacement. This ensures that all relations and entities are consistently and unambiguously represented across the query set. We keep fully expanded queries for future processing.

\noindent
\textbf{Seed Node Extraction:} After we have normalized the query workload, we extract the unique entities from the triple patterns that we can use as the seed nodes that can facilitate personalized summaries. For every unique query, we parse the triple patterns and locate URIs and useful tokens while omitting variables and literals. We use regular expressions in order to extract the URIs while filtering the SPARQL syntax.

\noindent
\textbf{Query Filtering:} We apply final filtering to the expanded query set to keep only the queries that mention one of the sampled seed nodes. This filtering step enables the summary construction to begin from the most appropriate and well-requested sections of the knowledge graph based on actual user queries. The filtered queries can be used to execute SPARQL queries against the KG and guide the development of the personalized summaries.

\noindent
\textbf{SPARQL Query Execution:} Preprocessed SPARQL queries are executed using Apache Jena Fuseki\footnote{\url{https://jena.apache.org/documentation/fuseki2/}}, an open-source SPARQL server, with knowledge graphs stored in the TDB2 backend for disk-based indexing and query optimization. Filtered queries are evaluated through Fuseki’s SPARQL endpoint using the ARQ engine, which performs internal indexing and query rewriting for efficient retrieval. Result sets are streamed and parsed to extract relevant triples for sensitivity computation in COREKG. This deployment supports concurrent execution, persistent endpoints, and low-latency access for large graphs such as Freebase, DBpedia, and Wikidata. This provides a scalable and reproducible integration between query preprocessing and coreset.

Each distinct seed set uniquely identifies a user, and all subsequent stages of the framework operate exclusively on the corresponding user-specific workload.

\begin{algorithm}[ht]
\caption{COREKG}
\label{alg:coreset_m}
\begin{algorithmic}[1]
\Require KG $G = (V, E, T)$; Query set $Q$; \#samples $m$
\Ensure Weighted summary $C \subseteq T$ with weights $w(t)$
\ForAll{$q \in Q$}
    \State $T_q \gets \{ t \mid \forall t \mbox{ relevant to query } q\}$
\EndFor
\ForAll{$t \in T$}
    \State $s(t) \gets \sum_{q \in Q} \frac{1}{|T_q|} \cdot \mathbb{I}[t \in T_q]$ \Comment{Triple sensitivity}
\EndFor
\State $S \gets \sum_{t \in T} s(t) = |Q|$
\ForAll{$t \in T$}
    \State $p(t) \gets \frac{s(t)}{S}$ \Comment{Sampling probability}
\EndFor
\State Sample $m$ triples iid from $T$ using distribution $p$
\State Let $C$ be the sampled triples
\ForAll{$t \in C$}
    \State $w(t) \gets \frac{S}{m \cdot s(t)} = \frac{|Q|}{m \cdot s(t)}$   
\EndFor
\State \Return subset $C$ and weight function $w$
\end{algorithmic}
\end{algorithm}

\subsection{Sensitivity Score}
As shown in Algorithm~\ref{alg:coreset_m} (Lines~4--6), COREKG computes a user-specific sensitivity score for each triple based on its contribution to the queries issued by an individual user. Let $u$ denote a user characterized by a seed entity set extracted from a query set.  All sensitivity computations are then performed with respect to filtered user-specific (i.e., personalized) workload $Q$. Formally, the sensitivity of a triple $t$ for a given user $u$ is defined as:
\[
s(t) = \sum_{q \in Q} \frac{1}{|T_q|} \cdot \mathbb{I}[t \in T_q],
\]
where $T_q$ denotes the set of triples relevant to query $q \in Q$, and $\mathbb{I}[t \in T_q]$ is an indicator function that returns 1 if $t$ contributes to the answer of $q$, and 0 otherwise. $T_q$ denotes the set of triples matched by the triple patterns of query $q$ under standard SPARQL evaluation. This definition ensures that the sensitivity of a triple is explicitly conditioned on the interests of user $u$, as encoded by their personalized query workload.
Intuitively, triples that appear frequently in a given user's queries receive higher sensitivity scores. Similarly, triples that occur in queries with small answer supports are emphasized and are therefore more likely to be included in the user's summary. As a result, different users induce different sensitivity distributions over the same underlying graph, leading to distinct personalized summaries. The total sensitivity for user $u$ is defined as $S = \sum_{t \in T} s(t)$.

The total sensitivity $S$ with respect to a fixed user workload $Q$, is bounded on $|Q|$.
This result simplifies normalization and guarantees that sensitivity-based sampling remains unbiased at the user level. Once sensitivities are computed, they are normalized into a user-specific sampling distribution. For each triple $t \in T$, the probability of sampling it for user $u$ is given by $p(t) = \frac{s(t)}{S}$.
This normalization ensures that the summary construction process is directly aligned with the query behaviour of each individual user rather than with the global query log.

We emphasize that sensitivity is not a universal measure of importance, but rather a quantity that varies depending on the user. Consequently, the same triple may be assigned different sensitivity values across users, depending on their respective query workloads. This user-conditioned notion of sensitivity constitutes the core mechanism through which personalization is realized in COREKG.

\subsection{Coreset Construction}
Given a user $u$ and the corresponding personalized query workload $Q$, we construct a coreset via sensitivity-based sampling, as detailed in Algorithm~\ref{alg:coreset_m} (Lines~11–15). Each sampled triple is assigned a weight to correct for sampling bias and ensure unbiased cost estimation. Algorithm~\ref{alg:coreset_m} is executed independently for each user using their respective workload $Q$, resulting in a separate, user-specific coreset. 

To correct for the sampling bias introduced by preferential selection, each sampled triple is assigned a weight:
$w(t) = \frac{1}{m \cdot p(t)} = \frac{S}{m \cdot s(t)}$
where $m$ is the sample size. This weighting guarantees that the cost computed from the coreset is unbiased, i.e., the expected cost equals the true cost of the full data. The weights also ensure bounded variance,    

COREKG conditions the sampling process on a user-specific query workload. For each user, the integration of filtered workload $Q$ enables COREKG to generate truly personalized summaries. This construction is repeated independently for each user, producing user-specific summaries with independent approximation guarantees.

\begin{theorem}[Coreset Size]{\label{theorem:main}}
Let $C$ and $w$ be the subset and weight function returned from the algorithm \ref{alg:coreset_m} over query workload \( Q \), and let \( X = \mathrm{cost}(G, Q) \) be the true cost on the full graph. Then, with probability at least \( 1 - \delta \) we have, $\left|\mathrm{cost}(C,Q) -\mathrm{cost}(G,Q) \right| \leq \epsilon \cdot \mathrm{cost}(G,Q)$
provided the number of samples \( m \) satisfies:
\[
m \geq \frac{8|Q|}{\epsilon^2} \cdot \log\left(\frac{1}{\delta}\right).
\]
\end{theorem}

The above theorem can be proved by applying Bernstein's inequality and sensitivity analysis. 

%% file: sections/Experimants.tex
\section{Evaluation}
%
We do an extensive empirical evaluation of real-world datasets to assess the quality of our summarization method. The evaluation is performed on real-world KGs with SPARQL queries. For evaluation, we divide queries into train and test with a ratio of 80:20. Our focus is to test the generality and scalability of our proposed approach. Query coverage is used as the evaluation metric. The implementation of this methodology is available at {\url{https://github.com/SohelKResearch/COREKG}.

\subsection{Experimental Design and Personalization}

All experiments are conducted in a personalized knowledge graph summarization setting. Each user is modeled by a profile defined as a set of seed entities. Each seed entity is extracted from the train query workload, which represents user interests. For each dataset, we construct 15 user profiles. Each user profile consists of 5 seed entities sampled uniformly at random from non-variable entities appearing in training queries. These seeds define a personalized workload by retaining only queries that mention at least one seed entity.
To ensure fair comparison, all baseline methods are constrained to produce summaries with the same size budget, measured as the number of retained triples. Our coreset sampling size is determined by theoretical bounds from coreset theory, guaranteeing approximation quality independent of the original KG size. This design ensures that observed performance differences stem from summarization quality rather than summary size. Reported results are averaged across users. COREKG performs user-specific preprocessing, including query filtering and workload construction, prior to sensitivity computation and coreset generation, ensuring fully personalized summaries. To ensure fair comparison, all models are evaluated exclusively on the same test queries associated with each user profile.

While COREKG is evaluated under static workload settings, the sensitivity formulation is additive over queries. This allows incremental updates as new queries arrive without recomputing from scratch. This shows that COREKG naturally supports streaming coresets. We focus on static workloads for fair comparison with prior methods such as GLIMPSE and iSummary.

\subsection{Datasets}

We evaluate COREKG on three widely used real-world knowledge graphs: Freebase, DBpedia, and Wikidata. These datasets differ substantially in scale, schema richness, and domain coverage. These datasets allow us to assess the robustness of personalized knowledge graph summarization under diverse and realistic query workloads. 

\noindent
\textbf{Freebase~\cite{bollacker2008freebase}\footnote{\url{https://developers.google.com/freebase}}:}
It is a large-scale knowledge base developed by Metaweb, containing a large number of RDF triples describing entities and their relationships across diverse domains. The KG Size is 1.9 billion with a query size of 4,737. For query workloads, we use SPARQL queries from the WebQSP dataset~\cite{yih2016value}, which consists of natural language questions paired with corresponding SPARQL queries over Freebase. In our experiments, we retain only the SPARQL queries. These queries exhibit non-trivial structures, including multiple triple patterns and joins, making them suitable for evaluating query-driven summarization methods.

\noindent
\textbf{DBpedia~\cite{mendes2012dbpedia}\footnote{\url{https://databus.dbpedia.org/dbpedia/collections/latest-core}}:}
It is a large-scale knowledge graph extracted from structured content in Wikipedia. We use the latest DBpedia core dataset, which contains a large number of RDF triples spanning a wide range of entity types and relations. The KG Size is 0.53 billion with a query size of 82,519. We use SPARQL query logs from the Linked SPARQL Queries (LSQ) dataset~\cite{stadler2024lsq}\footnote{\url{http://lsq.aksw.org/}}. The raw LSQ logs are processed using the \texttt{lsq-clean} toolkit\footnote{\url{https://github.com/sparqeology/lsq-clean?tab=readme-ov-file}} to remove malformed, duplicated, and template-generated queries. The resulting workload contains diverse query shapes, including star-shaped and multi-join patterns, reflecting realistic usage of DBpedia in practice.

\noindent
\textbf{Wikidata\footnote{\url{https://dumps.wikimedia.org/wikidatawiki/entities/}}:}
It is a large, collaboratively maintained knowledge graph curated by the Wikimedia community. The KG Size is 25.33 billion with a query size of 30,154. We use the latest full RDF dump, which contains a large number of triples describing entities, properties, and statements across a wide range of domains. For Wikidata query workloads, we use SPARQL queries from the LC-QuAD dataset 2.0, obtained from the publicly available release\footnote{\url{https://huggingface.co/datasets/mohnish/lc_quad/blob/main/data.zip}}. LC-QuAD provides complex question–query pairs designed to test reasoning over Wikidata. 


\subsection{Evaluation Metrics}

We assess summarization methods using two key metrics: structural coverage reflecting how well the summary preserves query patterns, and the F1 score evaluating answer‐retrieval accuracy.

\noindent
\textbf{Coverage:} Given a personalized summary $S$ for a user $s$, a query workload $Q$, and two weights $w_n$ and $w_p$ for nodes and edges, respectively, we define the coverage as:
{
\begin{align*}
\text{Coverage}(Q, S, s) = \frac{1}{n} \sum_{s \in q_i} \bigg(w_n \frac{\text{snodes}(S, q_i)}{\text{nodes}(q_i)} + \\ w_p \frac{\text{sedges}(S, q_i)}{\text{edges}(q_i)} \bigg)
\end{align*}
}

Here, $\text{nodes}(q_i)$ and $\text{edges}(q_i)$ represent the total number of nodes and edges in query $q_i$, while $\text{snodes}(S, q_i)$ and $\text{sedges}(S, q_i)$ denote the number of those nodes and edges that are also included in the summary $S$. In our setup, we assign equal importance to nodes and edges by setting $w_n = 0.5$ and $w_p = 0.5$.

\noindent
\textbf{F1 Score:}
The F1 score measures the query-answering accuracy of a summary $S$ by comparing its retrieved answers to the ground-truth answers obtained from the full knowledge graph. For each query, we compute true positives (TP) as answers returned by both $S$ and the full KG, false positives (FP) as answers returned by $S$ but not found in the full KG, and false negatives (FN) as answers present in the full KG but missing from $S$. From these, we calculate precision $P = \tfrac{\mathrm{TP}}{\mathrm{TP} + \mathrm{FP}}$ and recall $R = \tfrac{\mathrm{TP}}{\mathrm{TP} + \mathrm{FN}}$, and use them to derive the F1 score as $\mathrm{F1} = 2\,\frac{P\,R}{P + R}$.

Coverage and F1 Score are computed independently for each user-specific workload and then averaged across users.

\subsection{Baseline Comparison}

We compare COREKG with several baseline methods that vary in personalization strategies and their ability to preserve query-relevant information. Table~\ref{tab:combined_scores} reports F1 scores reflecting query-answering accuracy and workload coverage. Graph-proximity-based methods such as PPR and GLIMPSE rely on random walks or submodular selection over query logs; although they achieve moderate coverage in some cases, their low F1 scores indicate loss of essential relational structure, particularly for complex SPARQL queries under small summary budgets or sparse histories. iSummary improves scalability by learning from aggregate workloads but limits personalization, resulting in consistently low coverage and F1 scores. PEGASUS and APEX$^2$ improve coverage using heuristic node merging and adaptive heat diffusion. However, PEGASUS lacks formal guarantees and may lose fine-grained semantics, leading to inconsistent F1 scores on complex datasets such as Wikidata. APEX$^2$ assumes dense query activity and alignment between semantic relevance and graph proximity, which does not always hold in practice. In contrast, COREKG consistently achieves the highest F1 scores while maintaining strong coverage. It avoids reliance on structural heuristics and instead provides explicit approximation guarantees through sensitivity-based coreset sampling that selects triples according to their contribution to user-specific query workloads. As a result, COREKG delivers robust and controllable performance even under sparse, skewed, or evolving workloads.








 \begin{table}[ht]
\centering
\scriptsize
\setlength{\tabcolsep}{2.5pt}
\renewcommand{\arraystretch}{1.05}

\begin{tabular}{l|ccc|ccc}
\hline
& \multicolumn{3}{c|}{\textbf{F1(\%)}} 
& \multicolumn{3}{c}{\textbf{Coverage(\%)}} \\
\cline{2-7}

\textbf{Method} 
& Freebase & DBpedia & WikiData 
& Freebase & DBpedia & WikiData \\
\hline

COREKG 
& \cellcolor{best}\textbf{56.35}  & \cellcolor{best}\textbf{58.64} & \cellcolor{best}\textbf{52.10}
& \cellcolor{best}\textbf{66.92}   & \cellcolor{best}\textbf{68.07} & \cellcolor{best}\textbf{67.46}\\

APEX$^2$ 
& 49.76 & 47.52 & 45.29 
& 56.21 & 42.06 & 65.18 \\

iSummary 
& 21.83 & 19.38 & 30.64 
& 23.16 & 12.55 & 39.72 \\

PEGASUS  
& 3.18 & 2.46 & 2.90 
& 63.03 & 66.32 & 33.84 \\

GLIMPSE  
& 16.21 & 32.68 & 28.94 
& 31.39 & 34.62 & 11.93 \\

PPR      
& 8.64 & 11.39 & 9.75 
& 10.93 & 5.04 & 4.61 \\

\hline
\end{tabular}

\caption{\scriptsize{F1 and Coverage scores.}}
\label{tab:combined_scores}

\end{table}


To further analyse the behaviour of different methods under varying summary budgets, we report performance across multiple budget settings. COREKG consistently outperforms all baselines across the entire budget range. At a budget of 1000 triples (DBpedia), COREKG achieves 30.76\% coverage, and 30.12\% F1 and APEX$^2$ achieves 11.72\% coverage and 18.26\% F1. PEGASUS achieves 32.80\% coverage but only 1.72\% F1, indicating poor query relevance. At a budget of 30000 triples (DBpedia), COREKG achieves 66.72\% coverage and 57.83\% F1, while APEX$^2$ reaches 44.48\% coverage and 47.12\% F1. PEGASUS achieves 65.32\% coverage and 2.68\% F1 and shows negligible improvement beyond this point, indicating early convergence. In the table \ref{tab:combined_scores} we use a budget of 60000. Finally, at a budget of 100000 triples (DBpedia), COREKG achieves 74.8\% coverage and 60.72\% F1, while APEX$^2$ reaches 50.83\% coverage and 49.82\% F1, and PEGASUS has already converged earlier and does not improve further. Methods such as iSummary and GLIMPSE also show minimal gains beyond earlier budgets. 

Overall, COREKG consistently outperforms all baselines across budgets, demonstrating its ability to preserve query-relevant structure.
\subsection{Time \& Space Complexity Analysis}
The COREKG framework includes preprocessing stages such as query normalization, seed node extraction, and query filtering. Let $L$ denote the average characters in a query and $Q$ denote the number of queries. Query normalization and seed node extraction require scanning each query using regular expressions, resulting in $O(QL)$ time. The query filtering operates in $O(QL)$ time through URI token extraction and set-based intersection with the seed set, and the corresponding space complexity is $O(QL)$. SPARQL query execution is independent of COREKG and is externally handled using Apache Jena Fuseki. Thus, its time and space complexity depend on the underlying SPARQL engine, indexing strategy, and query structure. Sensitivity computation requires $O(TQ)$ with $O(T)$ space, where $T$ is the total number of triples. Probability normalization, sampling, and weight assignment require $O(T)$ time and $O(m)$ space, where $m$ is the coreset sample size.

\section{Ablation Study}
\label{app:ablation}
To evaluate the role of each component of COREKG, we perform an Ablation study in which we investigate the effect of each component on the overall performance while maintaining the same size summary.

\noindent
\textbf{First}, COREKG-Uniform, which replaces sensitivity-based sampling with uniform sampling, collapses to near-zero performance, achieving almost $0\%$ F1 across all datasets and negligible coverage. This confirms that uniform random sampling is ineffective for identifying relevant triples. This indicates the need to use query-based sensitivity to preserve meaningful query structure.

\noindent
\textbf{Second}, COREKG-Global, which removes user-specific conditioning and constructs summaries using a global workload, performs significantly worse than the full model. It achieves F1 scores of $39.78\%$ (Freebase), $44.06\%$ (DBpedia), and $37.63\%$ (Wikidata), with corresponding coverage scores of $47.34\%$ (Freebase), $51.24\%$ (DBpedia), and $49.17\%$ (Wikidata). COREKG-Global performed worse than the full COREKG model, demonstrating that to provide personalized query handling, we must include user-specific query behaviours as a fundamental component.

\noindent
\textbf{Finally}, COREKG-Unweighted treated each sample from the coreset equally and assigned each sample an equal weight. It shows a substantial degradation in performance.  COREKG-Unweighted achieving F1 scores of $10.82\%$ (Freebase), $11.53\%$ (DBpedia), and $19.38\%$ (Wikidata), and coverage scores of $14.06\%$ (Freebase), $16.68\%$ (DBpedia), and $14.72\%$ (Wikidata). This means that it could not adequately handle sampling and thus produce an overall very low coverage. It has been demonstrated through the ablation study that Weighted Coreset Sampling is effective in reducing sampling bias and preserving query semantics across workloads.

Based on our ablation analysis, we conclude that all elements of COREKG are important and complementary to its success.

\section{Limitations}

While COREKG demonstrates strong performance across datasets, several limitations remain. First, the effectiveness of sensitivity-based sampling depends on the representativeness of the user-specific query workload. When workloads are sparse, highly skewed, or dominated by a narrow set of seed entities, the resulting summaries may overfit to observed queries and fail to capture emerging or under-represented structural patterns via sensitivity analysis. 

The theoretical analysis in this work guarantees preservation of the aggregate workload cost across the query set. It does not provide guarantees for per-query answers. These properties depend on correlations among triples and are outside the scope of the current coreset formulation. Extending the framework to per-query or structure-aware guarantees is an important direction for future work.

\section{Conclusion}
COREKG generates high-quality personalized knowledge graph summaries using sensitivity-based sampling over user-specific query workloads. The summary preserves query-relevant structure while achieving high coverage and F1 scores. Overall, COREKG is a strong approach for personalized KG summarization.
\section*{Ethics and Genereative AI Statement}
This work uses publicly available knowledge graphs and SPARQL query workloads for research purposes. The proposed method performs personalized knowledge graph summarization and does not generate autonomous decisions. No generative AI tools were used in the development of the methodology, experiments, or analysis. Generative AI tools were used only for language refinement.
\section*{Acknowledgments}
The authors would like to acknowledge the partial support of the Infosys Center for Artificial Intelligence (CAI), IIIT-Delhi in this work. Supratim acknowledges the kind support from Anusandhan National Research Foundation (ECRG/2024/000959). The work was partly supported by their generous fund.

%% file: sections/Appendix.tex
\onecolumn
\section{Theoretical Analysis}\label{appendix:theory}

In the theoretical analysis, personalization is captured by fixing a single user and defining all variables with respect to that user’s query workload $Q$. The sensitivity scores $s(t)$, sampling probabilities $p(t)$, total sensitivity $S$, and the coreset cost functions are all computed solely with respect to this user-specific workload, rather than a global query set. As a result, the approximation guarantees, variance bounds, and coreset size results are derived on a per-user basis and hold independently for each personalized summary. This ensures that the user-specific construction of COREKG is fully aligned with the theoretical analysis.

Let the graph $G = (V, E, T)$ be a knowledge graph, where $V$ is the set of nodes, $E$ is the set of edges, and $T \subseteq V \times E \times V$ is the set of triples. Let $Q$ denote a user-specific query workload. For each query $q \in Q$, let $f_q(t) \in \{0, 1\}$ indicate whether a triple $t \in T$ is relevant to query $q$ by calculating the sensitivity score\label{definition:Sensitivity}. Using this, we define the total cost of the graph $G$ concerning the query set $Q$ as:
\begin{equation}
    \text{cost}(G,Q) = \sum_{q \in Q} \text{cost}(G,q) = \sum_{t \in T} \sum_{q \in Q} f_q(t)
\end{equation}
 
This cost function measures the total number of associations between triples in $T$ and queries in $Q$, which effectively counts the number of triples that are relevant to at least one query. A uniform weight \( w(t) = 1 \) is assigned to every triple in the full graph \( G \).

To reduce computational or storage overhead, our goal is to approximate $\text{cost}(G, Q)$ using a compact, weighted subset of the original graph. We define this subset as a \emph{coreset}, denoted $C = (V', E', T')$, where $V' \subseteq V$, $E' \subseteq E$, and $T' \subseteq T$ is a weighted subset of triples. Each triple $t \in T'$ is assigned a non-negative weight $w(t)$, representing its importance. The total cost of the coreset $C$ with respect to the query set $Q$ is defined as:
\begin{equation}
    \text{cost}(C,Q) = \sum_{t \in T'} \sum_{q \in Q} w(t) \cdot f_q(t)
\end{equation}

\begin{definition}[Coreset]
Let $G = (V, E, T)$ be a knowledge graph and let $\varepsilon \in (0, 1)$ be an approximation parameter. A weighted subset $C = (V', E', T')$ where $V' \subseteq V$ and $E' \subseteq E$ and $T' \subseteq T$, is called an \emph{$\varepsilon$-coreset} of $G$ with respect to a query workload $Q$ if, with probability at least \( 1 - \delta \), the following inequality holds:
\begin{equation}
    \left| \text{cost}(G, Q) - \text{cost}(C, Q) \right| \leq \varepsilon \cdot \text{cost}(G, Q)
\end{equation}
In other words, the coreset $C$ provides a $(1 \pm \varepsilon)$ multiplicative approximation to the total cost of the full graph $G$, ensuring that query-relevant information is preserved up to a small relative error.
\end{definition}

Instead of assigning a new function to model this, we interpret the existing cost formulation as inherently reflecting the weighted influence of each triple through its frequency and relevance across the query set. This perspective allows us to approximate $\text{cost}(G, Q)$ using a compact, weighted subset of triples selected via sensitivity-based sampling. The objective is to construct a coreset $C$ such that. 
\[
\text{cost}(C, Q) \approx \text{cost}(G, Q),
\]
Thereby enabling accurate and efficient summarization while preserving user-specific query semantics.

Now, we formally define the sensitivity score for every triple $t \in T$ that quantifies its importance with respect to the query workload $Q$.

\begin{lemma}[Bound on Total Sensitivity]\label{lemma:TotalSensitivity}
Let $G = (V, E, T)$ be a knowledge graph and $Q$ a query workload. 
For each triple $t \in T$, we define its \emph{sensitivity score} with respect to $Q$ as
\[
s(t) = \sum_{q \in Q} \frac{f_q(t)}{\mathrm{cost}(G,q)} = \sum_{q \in Q} \frac{1}{|T_q|} \cdot \mathbb{I}[t \in T_q],
\]
where $f_q(t) \in \{0,1\}$ is an indicator of whether triple $t$ is relevant to query $q$. Then, the total sensitivity across all triples is exactly the number of queries:
\[
S = \sum_{t \in T} s(t) = |Q|.
\]
\end{lemma}

\begin{proof}

Let,
\[
s(t) = \sum_{q\in Q}\frac{f_q(t)}{\mathrm{cost}(G,q)} = \sum_{q\in Q}\frac{\mathbb{I}[t\in T_q]}{\mathrm{cost}(G,q)}.
\]

By expanding $\mathrm{cost}(G,q)$:
\[
=\sum_{q\in Q}\frac{\mathbb{I}[t\in T_q]}{\displaystyle\sum_{t'\in T}\mathbb{I}[t'\in T_q]}
=\sum_{q\in Q}\frac{\mathbb{I}[t\in T_q]}{|T_q|}
\]
Therefore, 
\[
S = \sum_{t \in T} s(t) = \sum_{t \in T} \sum_{q \in Q} \frac{1}{|T_q|} \cdot \mathbb{I}[t \in T_q].
\]
We can exchange the order of summation:
\[
= \sum_{q \in Q} \sum_{t \in T} \frac{1}{|T_q|} \cdot \mathbb{I}[t \in T_q].
\]
For each query $q$, the inner sum over $t \in T$ counts exactly $|T_q|$ terms, each contributing $\frac{1}{|T_q|}$:
\[
\sum_{t \in T} \frac{1}{|T_q|} \cdot \mathbb{I}[t \in T_q] = \frac{|T_q|}{|T_q|} = 1.
\]
Thus, summing over all $q \in Q$ gives:
\[
\sum_{t \in T} s(t) = \sum_{q \in Q} 1 = |Q|.
\]
\end{proof}

\begin{lemma}[Expectation and Variance of Coreset Cost]\label{lemma:UnbiasVariance}
Let $G = (V, E, T)$ be a knowledge graph and $Q$ a query workload. 
For each triple $t \in T$, let $p(t)$ be its sampling probability, 
$X_t$ the number of times $t$ is sampled in $m$ independent rounds, 
and $w(t) = \tfrac{1}{m \cdot p(t)}$ its assigned weight. 
Define the cost contribution of $t$ as 
\[
|T_q| = \sum_{t \in T} f_q(t),
\]
where $f_q(t) \in \{0,1\}$ denotes whether $t$ is relevant to query $q$. 
Then, the following hold:

1. The expected cost of the coreset equals the cost of the full graph:
\[
\mathbb{E}[\mathrm{cost}(C,Q)] = \mathrm{cost}(G,Q).
\]

2. The variance of the coreset cost is bounded as:
\[
\operatorname{Var}[\mathrm{cost}(C,Q)] \leq \frac{|Q|}{m} \cdot \mathrm{cost}(G,Q)^2,
\]
\end{lemma}

\begin{proof}
\textbf{Expectation.} 
The coreset is formed by sampling $m$ triples with replacement, where each triple $t \in T$ is chosen with probability $p(t)$. Thus 
$X_t \sim \text{Binomial}(m, p(t))$. The coreset cost is
\begin{eqnarray*}
    \mathrm{cost}(C,Q) = \sum_{t \in T} \sum_{q \in Q} X_t \cdot w(t) f_{q}(t) = \sum_{t \in T} X_t \cdot \frac{1}{m \cdot p(t)} \cdot \sum_{q \in Q} f_q(t).
\end{eqnarray*}

Taking expectation and applying linearity:
\[
\mathbb{E}[\mathrm{cost}(C,Q)] 
= \sum_{t \in T} \frac{1}{m \cdot p(t)} \cdot \sum_{q \in Q} f_q(t) \cdot \mathbb{E}[X_t].
\]
we have $\mathbb{E}[X_t] = m \cdot p(t)$. Substituting:
\[
\mathbb{E}[\mathrm{cost}(C,Q)] 
= \sum_{t \in T} \sum_{q \in Q} f_q(t) 
= \mathrm{cost}(G,Q).
\]

\textbf{Variance.}
%


 %
Lets consider $cost(G,Q)= \sum_{t \in T} c_t$ is the total query cost for a triple and $r_i = \sum_{t \in T} y_{ti} \cdot \frac{c_t}{m \cdot p(t)}$ be the per-round cost contribution, where $y_{ti} \in \{0,1\}$ indicates selection of triple $t$ in round $i$. The total coreset cost is:

\[
\mathrm{cost}(C,Q) = \sum_{i=1}^m r_i
\]

The expectation of each $r_i$ is:
\begin{align*}
\mathbb{E}[r_i]
=
\mathbb{E}\left[
\sum_{t\in T}
y_{ti}\cdot \frac{c_t}{m\,p(t)}
\right] =
\sum_{t\in T}
\mathbb{E}\left[
y_{ti}\cdot \frac{c_t}{m\,p(t)}
\right]=
\sum_{t\in T}
\frac{c_t}{m\,p(t)}
\cdot
\mathbb{E}[y_{ti}]  =
\sum_{t\in T}
\frac{c_t}{m\,p(t)}
\cdot p(t) =
\sum_{t\in T}
\frac{c_t}{m} =
\frac{\mathrm{cost}(G,Q)}{m}.
\end{align*}

For the variance decomposition:
\begin{align*}
\operatorname{Var}[r_i] = \mathbb{E}[r_i^2] - (\mathbb{E}[r_i])^2 = \frac{1}{m^2}\sum_{t \in T} \frac{c_t^2}{p(t)} - \frac{\text{cost}(G, Q)^2}{m^2}
\end{align*}

The total variance becomes:
\[
\operatorname{Var}[\mathrm{cost}(C,Q)] = m \cdot \operatorname{Var}[r_i] \leq \frac{1}{m}\sum_{t \in T} \frac{c_t^2}{p(t)}
\]

Substituting the sensitivity-based sampling probability $p(t) = \frac{s(t)}{S}$:
\[
\operatorname{Var}[\mathrm{cost}(C,Q)] \leq \frac{S}{m}\sum_{t \in T} \frac{c_t^2}{s(t)}
\]

Let $L_{\max} = \max_{q\in Q}|T_q|$ denote the maximum number of relevant triples associated with any query in the workload.

Since $|T_q| \leq L_{\max}$ for every query $q \in Q$, we have $\frac{1}{|T_q|}\geq\frac{1}{L_{\max}}$ Substituting this inequality into the definition of $s(t)$ gives
\begin{align*}
s(t)
=
\sum_{q\in Q}\frac{1}{|T_q|}f_q(t) \geq \sum_{q\in Q}\frac{1}{L_{\max}}f_q(t) =
\frac{1}{L_{\max}}
\sum_{q\in Q}f_q(t) =
\frac{c_t}{L_{\max}}.
\end{align*}

Rearranging,
\[
\frac{1}{s(t)}
\leq
\frac{L_{\max}}{c_t}.
\]
Multiplying both sides by $c_t^2$ yields
\[
\frac{c_t^2}{s(t)}
\leq
L_{\max}c_t.
\]

Summing over all triples $t \in T$,
\begin{align*}
\sum_{t\in T}\frac{c_t^2}{s(t)}
\leq
L_{\max}\sum_{t\in T}c_t \leq
(\mathrm{cost}(G,Q))^{2}.
\end{align*}

Since, $L_{\max} = \max_{q \in Q}|T_q| \leq \sum_{q \in Q}|T_q| = \mathrm{cost}(G,Q)$. Now substituting this bound into the variance expression,
\begin{align*}
\operatorname{Var}[\mathrm{cost}(C,Q)]
\leq
\frac{S}{m}\mathrm{cost}(G,Q)^2 
\end{align*}
\end{proof}

\begin{theorem}[Coreset Size]{\label{theorem:appendix}}
Let $C$ and $w$ be the subset and weight function returned from the algorithm \ref{alg:coreset_m} over query workload \( Q \), and let \( X = \mathrm{cost}(G, Q) \) be the true cost on the full graph. Then, with probability at least \( 1 - \delta \) we have, $\left|\mathrm{cost}(C,Q) -\mathrm{cost}(G,Q) \right| \leq \epsilon \cdot \mathrm{cost}(G,Q)$
provided the number of samples \( m \) satisfies:
\[
m \geq \frac{8|Q|}{\epsilon^2} \cdot \log\left(\frac{1}{\delta}\right).
\]
\end{theorem}

\begin{proof}
We express the total coreset cost as a sum of independent contributions:
\[
\mathrm{cost}(C,Q) = \sum_{i=1}^m r_i
\]
where each \( r_i = \frac{c_t}{m p(t)} \) represents the contribution of $i^{th}$ sample which samples the triple \( t \). 

We first bound the deviation of the random variables $|r_i - \mathbb{E}[r_i]|$. From the Lemma \ref{lemma:UnbiasVariance} we have,
\[
\mathbb{E}[r_i]
=
\frac{1}{m}\sum_{t\in T}c_t
=
\frac{\mathrm{cost}(G,Q)}{m}.
\]

Therefore,
\begin{align*}
\left|r_i-\mathbb{E}[r_i]\right|
=
\left|
\frac{c_t}{m\cdot p(t)}
-
\frac{\mathrm{cost}(G,Q)}{m}
\right| 
=
\frac{1}{m}
\left|
\frac{c_t}{p(t)}-\mathrm{cost}(G,Q)
\right| \leq
\frac{1}{m}\cdot \frac{c_t}{p(t)} = 
\frac{S}{m}\cdot \frac{c_t}{s(t)}
\end{align*}

Since $f_q(t)$ and $\mathrm{cost}(G,Q)$ are non-negatives, so we have the following, 
$$s(t) = \sum_{q \in Q} \frac{f_q(t)}{\mathrm{cost}(G,q)} \geq \frac{\sum_{q \in Q} f_q(t)}{\sum_{q \in Q} \mathrm{cost}(G,q)} = \frac{c_t}{\mathrm{cost}(G,Q)}.$$

So, we get, $| r_i - \mathbb{E}[r_i] | \leq b $ where $b = \frac{S}{m} \mathrm{cost}(G,Q)$. Furthermore, we know that, $\operatorname{Var}[\mathrm{cost}(C,Q)] \leq \frac{S}{m} \text{cost}(G, Q)^2$.

Now, applying Bernstein's inequality:
\begin{eqnarray*}
    \Pr\left( \left|\mathrm{cost}(C,Q) - \mathrm{cost}(G,Q)\right| > \epsilon \mathrm{cost}(G,Q) \right) &\leq& \exp\left( \frac{-\epsilon^2 \mathrm{cost}(G,Q)^2}{2 \cdot \operatorname{Var}[\mathrm{cost}(C,Q)] + \frac{2}{3} b \epsilon \mathrm{cost}(G,Q)} \right) \\
    &=& \exp\Bigg(\frac{-\epsilon^2 \, \mathrm{cost}(G,Q)^2}{\frac{2 S \, \mathrm{cost}(G,Q)^2}{m} + \frac{2 S \epsilon \, \mathrm{cost}(G,Q)^2}{3m}}\Bigg) \\
    &=& \exp\left( \frac{-\epsilon^2 m}{2S + \frac{2}{3} \epsilon S} \right)
\end{eqnarray*}



To ensure this probability is at most \( \delta \), we require:
\[
\exp\left( \frac{-\epsilon^2 m}{2S + \frac{2}{3} \epsilon S} \right) \leq \delta
\Rightarrow
m \geq \frac{2S + \frac{2}{3} \epsilon S}{\epsilon^2} \cdot \log\left( \frac{1}{\delta} \right)
\]

Finally, noting that \( 2 + \frac{2}{3} \epsilon \leq 8 \) for \( \epsilon \in (0,1) \), we conclude:
\[
m \geq \frac{8 S}{\epsilon^2} \cdot \log\left( \frac{1}{\delta} \right)
\]

\noindent Substituting \(S = |Q|\), we get:
\[
m \geq \frac{8 |Q|}{\epsilon^2} \cdot \log\left( \frac{1}{\delta} \right).
\]
\end{proof}

This follows from Bernstein's inequality since for any $\varepsilon \in (0,1)$ and $\delta \in (0,1)$ choosing $m = O\!\left(\tfrac{S}{\varepsilon^{2}}\log \tfrac{1}{\delta}\right)$ guarantees with probability at least $1-\delta$ that the weighted coreset is a $(1\pm\varepsilon)$ approximation to the workload and hence the summary size depends on $S$ (or equivalently $|Q|$) rather than $|T|$.